%% file: main.tex
\PassOptionsToPackage{nameinlink}{cleveref}
\documentclass[final,12pt,cleveref]{colt2024} 

\usepackage{times}

\usepackage[utf8]{inputenc}
\usepackage{hyperref}
\usepackage{bm}
\usepackage{csquotes}
\usepackage{anyfontsize}
\usepackage[shortlabels]{enumitem}

\usepackage{algorithm}
\usepackage{algorithmic}
\usepackage{xspace}
\usepackage{natbib}

\newaliascnt{fact}{theorem}
\newtheorem{fact}[fact]{Fact}
\aliascntresetthe{fact}
\crefname{fact}{fact}{facts}

\newaliascnt{condition}{theorem}
\newtheorem{condition}[condition]{Condition}
\aliascntresetthe{condition}
\crefname{condition}{condition}{conditions}

\newcommand{\cX}{\mathcal{X}}
\newcommand{\cM}{\mathcal{M}}

\newcommand{\ind}{\mathbf{1}}
\newcommand{\eps}{\varepsilon}
\newcommand{\N}{\mathbb{N}}

\newcommand{\R}{\mathbb{R}}

\newcommand{\E}{\mathbb{E}}

\newcommand{\cP}{\mathcal{P}}

\newcommand{\bp}{\mathbf{p}}

\newcommand{\cL}{\mathcal{L}}
\renewcommand{\phi}{\varphi}

\renewcommand{\setminus}{\smallsetminus}

\newcommand{\tg}{\tilde{g}}

\newcommand{\tO}{\tilde{O}}

\newcommand{\cN}{\mathcal{N}}
\newcommand{\cW}{\mathcal{W}}
\newcommand{\cY}{\mathcal{Y}}
\newcommand{\hw}{\hat{w}}
\newcommand{\A}{\mathrm{A}}
\newcommand{\argmax}{\mathrm{argmax}}
\newcommand{\cB}{\mathcal{B}}

\newcommand{\priv}{\mathrm{priv}}
\newcommand{\pub}{\mathrm{pub}}
\newcommand{\clip}{\mathrm{clip}}
\newcommand{\trunc}{\mathrm{trunc}}

\DeclareMathOperator{\Lap}{Lap}

\newcommand{\hybriddp}{\text{semi-sensitive DP}\xspace}
\newcommand{\hybriddpuc}{\text{Semi-Sensitive DP}\xspace}
\newcommand{\hybridcdp}{\text{semi-sensitive zCDP}\xspace}
\newcommand{\fulldp}{\text{full DP}\xspace}
\newcommand{\polylog}{\mathrm{polylog}}

\newcommand{\prn}[1]{\left( #1 \right)}
\newcommand{\ang}[1]{\left\langle #1 \right\rangle}

\newenvironment{proofof}[1][\unskip]%
  {%
   \par\noindent{\bfseries\upshape \proofname\ of #1}%
  }%
  {\jmlrQED}

\allowdisplaybreaks

\title[On Convex Optimization with Semi-Sensitive Features]{On Convex Optimization with Semi-Sensitive Features}

\coltauthor{%
 \Name{Badih Ghazi} \Email{badihghazi@gmail.com}\\
 \addr Google Research, Mountain View, USA
 \AND
 \Name{Pritish Kamath} \Email{pritish@alum.mit.edu}\\
 \addr Google Research, Mountain View, USA
 \AND
 \Name{Ravi Kumar} \Email{ravi.k53@gmail.com}\\
 \addr Google Research, Mountain View, USA
 \AND
 \Name{Pasin Manurangsi} \Email{pasin@google.com}\\
 \addr Google Research, Thailand
 \AND
 \Name{Raghu Meka} \Email{raghum@cs.ucla.edu}\\
 \addr University of California, Los Angeles, USA
 \AND
 \Name{Chiyuan Zhang} \Email{chiyuan@google.com}\\
 \addr Google Research, Mountain View, USA
}

\begin{document}

\maketitle

\begin{abstract}
We study the differentially private (DP) empirical risk minimization (ERM) problem under the \emph{\hybriddp} setting where only some features are sensitive. This generalizes the Label DP setting where only the label is sensitive. We give improved upper and lower bounds on the excess risk for DP-ERM. In particular, we show that the error only scales polylogarithmically in terms of the sensitive domain size, improving upon previous results that scale polynomially in the sensitive domain size~\citep{labeldp}.
\end{abstract}



\input{s1-intro}
\input{s2-prelim}
\input{s3-mwu-newer}
\input{s4-conclusion}

\newpage
\bibliography{ref}
\newpage
\appendix
\input{a1-app}
\input{a2-lb}

\end{document}

%% file: s1-intro.tex
\section{Introduction}

In empirical risk minimization (ERM) problem, we are given a dataset $D = \{x_i\}_{i \in [n]} \in \cX^n$ and a loss function $\ell: \cW \times \cX \to \R$ and the goal is to find $w \in \cW$ that minimizes the empirical risk $\cL(w; D) := \frac{1}{n} \sum_{i \in [n]} \ell(w; x_i)$. The excess risk is defined as $\cL(w; D) - \min_{w' \in \cW} \cL(w'; D)$.
Often, the dataset might contain sensitive data, and to provide privacy protection, we will use the notion of differential privacy (DP)~\citep{dwork2006calibrating}. ERM is among the most well-studied problems in the DP literature and tight excess risk bounds are known under assumptions such as Lipschitzness, convexity, and strong convexity (e.g.,~\cite{ChaudhuriMS11,KiferST12,BST14,WangYX17,BassilyFTT19,FeldmanKT20,GopiLL22}). 

In most of these studies, each $x_i$ is assumed to be sensitive. However, in several applications, such as online advertising, it can be the case that $x_i$ consists of both sensitive and non-sensitive attributes. This can be modeled by letting $\cX = \cX^{\pub} \times \cX^{\priv}$ where $\cX^{\pub}$ is the domain of the {\em non-sensitive} features and $\cX^{\priv}$ is the domain of the {\em sensitive} features. We will also write each example $x$ as $(x^{\pub}, x^{\priv})$ where $x^{\pub} \in \cX^{\pub}$ and $x^{\priv} \in \cX^{\priv}$. Here, our only aim is to protect $x^{\priv}$; in terms of DP, this means that we allow two neighboring datasets to differ only on the sensitive features of a single example. We refer to this DP notion as \emph{\hybriddp}. (See \Cref{sec:prelim-dp} for a more formal definition.) Our definition is identical to the ones considered by \cite{chua24training,shen23classification}; a related notion has been considered recently as well~\citep{krichene23private}. To avoid confusion, we refer to the standard DP notion (where a single entire example can be changed in neighboring datasets) as \emph{\fulldp}. Throughout, we use $k$ to denote the size of the domain for private features, i.e., $k = |\cX^{\priv}|$.

The \hybriddp\ model generalizes the so-called \emph{label DP}
~\citep{GKKLMVZ23} where the only sensitive ``feature'' is the label.\footnote{In our formulation of ERM, there is no distinction between a label and a feature; indeed, the two models are equivalent.} In our language,~\cite{labeldp} give an $\eps$-\hybriddp algorithm for convex ERM (under Lipschitzness assumption) that yields an expected excess risk of $\tO\left(\frac{k}{\eps \sqrt{n}}\right)$; for $(\eps, \delta)$-\hybriddp, they achieve an expected excess risk of $\tO\left(\frac{\sqrt{k \log(1/\delta)}}{\eps \sqrt{n}}\right)$. Complementing these upper bounds, they also provide a lower bound of $\Omega\left(\frac{1}{\eps\sqrt{n}}\right)$ against any $(\eps,\delta)$-\hybriddp. An interesting aspect of these bounds is that they are dimension-independent; meanwhile, for full DP, it is known that the expected excess risk grows (polynomially) with the dimension of $\cW$~\citep{BST14}. Despite this, the results from~\cite{labeldp} leave a rather large gap in terms of $k$: the upper bounds have a polynomial dependence on $k$ that is not captured by the lower bound.

\subsection{Our Contributions}

The main contribution of our paper is to (nearly) close this gap. In particular, we show that the dependency on $k$ is polylogarithmic rather than polynomial, as stated below.

\begin{theorem}[Informal; see \Cref{thm:convex-main}]
For $\eps \leq O(\log 1/\delta)$, there is an $(\eps, \delta)$-\hybriddp algorithm for ERM w.r.t. any $G$-Lipschitz convex loss function with domain radius $R$ that has expected excess empirical risk $\tO\left(RG \cdot \frac{\sqrt[4]{\log(1/\delta) \cdot \log k}}{\sqrt{\eps n}}\right)$.
\end{theorem}
\begin{theorem}[Informal; see \Cref{thm:lb-convex}]
For $\eps \leq O(\log k)$, any $(\eps, o(1/k))$-\hybriddp algorithm for ERM w.r.t. any $G$-Lipschitz convex loss function with domain radius $R$ has expected excess empirical risk at least $\Omega\left(RG \cdot \min\left\{1, \frac{\sqrt{\log k}}{\sqrt{\eps n}}\right\}\right)$.
\end{theorem}

Notice that the dependency on $k$ is essentially tight for $\delta = 1 / k^{1 + \Theta(1)}$. It remains an interesting open question to tighten the bound for a wider regime of $\delta$ values.

When the loss function is further assumed to be strongly convex and smooth, we can improve on the above excess risk and also provide a nearly tight bound in this case.

\begin{theorem}[Informal; see \Cref{thm:strongly-convex-smooth-main}]
For $\eps \leq O(\log 1/\delta)$, there is an $(\eps, \delta)$-\hybriddp algorithm for ERM w.r.t. any $G$-Lipschitz $\mu$-strongly convex $\lambda$-smooth loss function that has expected excess empirical risk $\tO\left(\frac{G^2}{\mu} \cdot \frac{\sqrt{\log(1/\delta) \cdot \log k} \cdot \log(\lambda / \mu)}{\eps n}\right)$.
\end{theorem}

\begin{theorem}[Informal; see \Cref{thm:lb-strongly-convex}]
For $\eps \leq O(\log k)$, any $(\eps, o(1/k))$-\hybriddp algorithm for ERM w.r.t. any $G$-Lipschitz $\mu$-strongly convex $\mu$-smooth loss function has expected excess empirical risk at least $\Omega\left(\frac{G^2}{\mu} \cdot \min\left\{1, \frac{\log k}{\eps n}\right\}\right)$.
\end{theorem}
\noindent Finally, our techniques are sufficient for solving \emph{multiple} convex ERM problems on the same input dataset, where the error grows only polylogarithmic in the number of ERM problems:
\begin{theorem}[Informal; see \Cref{thm:convex-main}]
For $\eps \leq O(\log 1/\delta)$, there is an $(\eps, \delta)$-\hybriddp algorithm for $m$ ERM problems w.r.t. any $G$-Lipschitz convex loss function with domain radius $R$ that has expected excess empirical risk $\tO\left(RG \cdot \frac{\sqrt[4]{\log(1/\delta) \cdot \log k} \sqrt{\log m}}{\sqrt{\eps n}}\right)$.
\end{theorem}
\begin{theorem}[Informal; see \Cref{thm:strongly-convex-smooth-main}]
For $\eps \leq O(\log 1/\delta)$, there is an $(\eps, \delta)$-\hybriddp algorithm for $m$ ERM problems w.r.t. any $G$-Lipschitz $\mu$-strongly convex $\lambda$-smooth loss function that has expected excess empirical risk $\tO\left(\frac{G^2}{\mu} \cdot \frac{\sqrt{\log(1/\delta) \cdot \log k} \cdot \log(m\lambda / \mu)}{\eps n}\right)$.
\end{theorem}
\noindent In the full DP setting, this problem was first studied by~\cite{Ullman15}. The error bound was improved by \cite{FeldmanGV17}, but their bound still depends polylogarithmically on the dimension $d$. By removing this dependency, our theorem above improves upon the bounds of \cite{FeldmanGV17}. We note, however, that \cite{FeldmanGV17} also give bounds for the $\ell_p$-bounded setting for any $p \ne 2$, but, for simplicity, we do not consider this in our work.

\subsection{Technical Overview}

In this section, we briefly discuss the techniques used in our work.

\paragraph{Answer Linear Vector Queries.}
The key ingredient of our work is an algorithm that can answer online linear \emph{vector} queries. Such a query is of the form $f: \cX \to \cB_2^d(1)$ where $\cB_2^d(1)$ denotes the (Euclidean) unit ball in $d$ dimensions, and our goal is to approximate $f(D) := \frac{1}{n} \sum_{i \in [n]} f(x_i)$. There can be up to $T$ online queries (i.e., we have to answer the previous query before receiving the next).

The case $d = 1$ is often referred to as linear queries. In this case, in the \fulldp setting,~\cite{HardtR10} introduced the ``Private Multiplicative Weights'' algorithm that has error $\polylog(|\cX|, \frac{1}{\delta}) / \sqrt{\eps n}$. We extend their algorithm in two crucial aspects:

(i) We adapt the algorithm to the \hybriddp setting and show that we can improve on the error in this setting: the $|\cX|$ term (size of the entire input domain) becomes $k = |\cX^{\priv}|$ (size of the sensitive features domain). 

(ii) We show a natural way to handle $d > 1$. In this case, the error is now measured in the $\ell_2$-error of the vector. Interestingly, we show that the error remains (roughly) the same in this setting and is in fact dimension-independent. This is crucial for achieving dimension-independent bounds in our theorems. 

The algorithm of \cite{HardtR10} works by maintaining a distribution over all the domain $\cX$. For each query, we (privately) check whether the current distribution is sufficiently accurate to answer the query. If so, we answer using the current distribution. Otherwise, we apply a multiplicative weight update (MWU) rule to update the distribution. The MWU rule depends on the privatized true answer and the answer computed using the current distribution, where the former is achieved via, e.g., adding Laplace noise. The crux of the analysis is that the privacy budget is only charged when an update occurs. Finally, a standard analysis of MWU shows that there cannot be too many updates.

To achieve (i), our algorithm maintains, for each example $x_i$, a distribution over $x_i^{\priv}$ and applies a multiplicative weight update. For (ii), we modify the update as follows. First, we privatize the true answer using the Gaussian mechanism. Then, we apply the MWU rule based on the dot product of this privatized true answer and the value of each example. Crucially, our analysis shows that, even though the total norm of the noise can be very large (growing with the dimension), it does not interfere too much with the update as only the noise in a few directions is relevant.   

\paragraph{From Answering Linear Vector Queries to Convex Optimization.}
By letting each query $f$ be the gradient of the loss function, our aforementioned algorithm allows one to construct an approximate gradient oracle. By leveraging existing results in the optimization literature \citep{dAspremont08,DevolderGN14,devolder2013first}, we immediately arrive at the claimed bounds.

\paragraph{Comparison to Previous Work.}
\cite{FeldmanGV17} show that approximate gradient oracle can be accomplished via Statistical Queries (SQs). For the purpose of our high-level discussion, one can think of SQs as just linear queries. Using this, they observe that the \cite{HardtR10} algorithm can be used to solve convex optimization problem(s) with low error. We note that this approach can be used in our setting, too, once we extend the Hardt--Rothblum algorithm to the \hybriddp setting with decreased error (i). However, this alone does \emph{not} yield a dimension-independent bound since the number of linear queries required still depends on the dimension. As such, we still require (ii) to achieve the results stated here. Finally, we remark that vector versions of MWU have been used in the DP literature before (e.g.,~\cite{Ullman15}). However, we are not aware of its study with respect to the effect of Gaussian noise; in particular, to the best of our knowledge, the fact that we still have a dimension-independent bound even after applying noise is novel.

\paragraph{Lower Bounds.} Suppose for simplicity that $\eps = \Theta(1)$. For the lower bound, we first recall the construction from previous work \citep{labeldp}, which is a reduction from (vector) mean estimation. Roughly speaking, they let the $i$th example contribute only to the $i$th coordinate and let the sensitive feature (which is binary in \cite{labeldp}) determine whether this coordinate should be +1 or -1. They argue that any $(\eps,\delta)$-\hybriddp algorithm must make an error in determining the sign of $\Omega(1)$ fraction of the coordinates; this results in the $\Omega\left(\frac{1}{\sqrt{n}}\right)$ error for mean estimation, which can then be converted to a lower bound for convex ERM via standard techniques~\citep{BST14}. We extend this lower bound by grouping together $O\left(\log k\right)$ examples and assign $k$ common coordinates for them. The examples in each group share the same sensitive feature, and it determines which of the $k$ coordinates the examples contribute to. In other words, each group is a hard instance of the so-called selection problem. This helps increase the error to $\Omega\left(\sqrt{\frac{\log k}{n}}\right)$.

%% file: s2-prelim.tex
\section{Preliminaries}
\label{sec:prelim}

We use $\cB_2^d(R)$ to denote the Euclidean ball of radius $R$ in $d$ dimensions, i.e., $\{y \in \R^d \mid \|y\|_2 \leq R\}$.

\subsection{Differential Privacy}
\label{sec:prelim-dp}

We recall the definition of differential privacy below.

\begin{definition}[Differential Privacy,~\cite{dwork2006calibrating}] \label{def:dp}
For $\eps, \delta \geq 0$, a mechanism $\A$ is said to be $(\eps, \delta)$-differentially private ($(\eps, \delta)$-DP) with respect to a certain neighboring relationship iff, for every pair $D, D'$ of neighboring datasets and every set $S$ of outputs, we have $\Pr[\A(D) \in S] \leq e^\eps \cdot \Pr[\A(D') \in S] + \delta$.
\end{definition}

In this paper, we consider datasets consisting of examples with a sensitive and non-sensitive part. More precisely, each dataset $D$ is $\{x_i\}_{i \in [n]}$ where $x_i = (x^{\pub}_i, x^{\priv}_i) \in \cX^{\pub} \times \cX^{\priv}$. Two datasets $D = \{(x^{\pub}_i, x^{\priv}_i)\}_{i \in [n]}, D' = \{(x'^{\pub}_i, x'^{\priv}_i)\}_{i \in [n]}$ are \emph{neighbors} if they differ on a single example's sensitive part. I.e., $x^{\pub}_i = x'^{\pub}_i$ for all $i \in [n]$ and there exists $i' \in [n]$ such that $x^{\priv}_i = x'^{\priv}_i$ for all $i \in [n] \setminus \{i'\}$. We often use the prefix ``semi-sensitive'' (e.g., \hybriddp) to signify that we are working with this neighboring relationship notion. Note that, the lemmas below that are stated without such a prefix, hold for any neighboring relationship.



For the purpose of privacy accounting, it will be convenient to work with the zero-concentrated DP (zCDP) notion.

\begin{definition}[\cite{DworkR16,BunS16}]
For $\rho > 0$, an algorithm $\A$ is said to be \emph{$\rho$-zero concentrated DP ($\rho$-zCDP)} with respect to a certain neighboring relationship iff, for every pair $D, D'$ of neighboring datasets and every $\alpha > 1$, we have $D_{\alpha}(\A(D) \| A(D')) \leq \rho \cdot \alpha$, where $D_{\alpha}(P \| Q)$ denotes the $\alpha$-Renyi divergence between $P$ and $Q$.%
\end{definition}

We will use the following results from \citet{BunS16} in the privacy analysis.

\begin{lemma}[\boldmath $\rho$-zCDP vs $(\eps, \delta)$-DP] \label{lem:dp-cdp}
(i) For any $\eps > 0$, any $\eps$-DP mechanism is $(0.5\eps^2)$-zCDP.
(ii) For any $\rho > 0$ and $\delta \in (0, 1/2)$, a $\rho$-zCDP mechanism is $\left(\rho + 2\sqrt{\rho\ln(1/\delta)}, \delta\right)$-DP.
\end{lemma}

\begin{lemma}[zCDP composition] \label{lem:composition}
If $\cM$ is a mechanism is a (possibly adaptive) composition of mechanisms $\cM_1, \dots, \cM_T$, where $\cM_i$ is $\rho_i$-zCDP, then $\cM$ is $(\rho_1 + \cdots + \rho_T)$-zCDP.
\end{lemma}

\subsection{Assumptions on the Loss Function}

Throughout this work, we assume that the loss function $\ell$ is convex and subdifferentiable (in the first parameter). Furthermore, we assume that it is $G$-Lipschitz; that is, $|\ell(w) - \ell(w')| \leq G \cdot \|w - w'\|_2$. 

There are also two additional assumptions that we use in our second result (\Cref{thm:strongly-convex-smooth-main}):
\begin{itemize}[nosep]
\item $\mu$-strong convexity: $\ell(w) \geq \ell(w') + \left<\nabla \ell(w'), w- w' \right> + \frac{\mu}{2} \|w - w'\|_2^2$.
\item $\lambda$-smoothness: $\nabla \ell$ is $\lambda$-Lipschitz, implying, $\ell(w) \leq \ell(w') + \left<\nabla \ell(w'), w- w' \right> + \frac{\lambda}{2} \|w - w'\|_2^2$.
\end{itemize}

\subsection{Concentration Bounds}
\label{sec:prelim-concen}

We will now prove a lemma with respect to a ``clipped'' distribution. To do this, let us define the clipping operation as follows. For $\phi \in \R^d$ and $c \in \R_{> 0}$, we let $\clip_{\phi,c}: \R^d \to \R^d$ be defined as\footnote{In other words, $u$ is scaled so that its $\phi$-semi-norm is at most $c$.}
\begin{align*}
\clip_{\phi,c}(u) =
\begin{cases}
u \cdot \min\{1, c / |\left<\phi, u\right>|\} &\text{ if } \left<\phi, u\right> \ne 0 \\
u & \text{ if } \left<\phi, u\right> = 0,
\end{cases}
\end{align*}
For convenience, for $c > 0$, we also define $\trunc_c: \R \to \R$ to denote\footnote{Note that this coincides with $\clip_{1, c}$ but we keep a separate notation for brevity and clarity.} the function $\trunc_c(b) := b \cdot \min\{1, c/|b|\}$, i.e., a rescaling of $b$ so that its absolute value is at most $c$.

The desired lemma is stated below. Although it might seem overly specific at the moment, we state it in this form as it is most convenient for our usage in the accuracy analysis later (without specifying too many extra parameters). Its proof is deferred to Appendix~\ref{app:concen-proof}.

\begin{lemma} \label{lem:clip-concen}
Let $\cP$ be any distribution over $\cB_2^d(1)$ and $\mu_{\cP} := \E_{U \sim \cP}[U]$. Let $Z$ be drawn from $\cN(\mu_Z, \sigma_Z^2 I_d)$ for some $\sigma_Z \in (0, 1], \mu_Z \in \cB_2^d(2)$. Then, we have
\begin{align*}
\Pr_Z\left[\left|\left<Z, \E_{U \sim \cP}[\clip_{Z, 3}(U)] - \mu_{\cP}\right>\right| > 2 \exp(-0.1/\sigma_Z^2)\right] < 2 \exp(-0.1/\sigma_Z^2).
\end{align*}
\end{lemma}

%% file: s3-mwu-newer.tex
\newcommand{\mwu}{\mathrm{MWU}}
\newcommand{\itermax}{L_{\max}}
\newcommand{\tlambda}{\tilde{\lambda}}
\newcommand{\tmu}{\tilde{\mu}}
\newcommand{\tF}{\tilde{F}}
\newcommand{\normbound}{\iota}
\newcommand{\ophi}{\overline{\phi}}

\section{Answering Linear Vector Queries with \hybriddpuc}

As mentioned in the introduction, we consider a setting similar to \cite{HardtR10} but with two main changes: (i) we support \hybriddp and (ii) each query in the family is allowed to be vector-valued (instead of scalar-valued). We describe this setting in more detail below.

A (bounded $\ell_2$-norm) \emph{linear vector query} is a function $f: \cX \to \cB_2^d(1)$, where $d \in \N$. The value of the function on a dataset $D = \{x_i\}_{i \in [n]}$ is defined as $f(D) := \frac{1}{n} \sum_{i \in [n]} f(x_i)$.

\paragraph{Online Linear Vector Query problem.}
In the \emph{Online Linear Vector Query (OLVQ)} problem, the interaction proceeds in $T$ rounds. At the beginning, the algorithm receives the dataset $D$ as the input. In round $t$, the analyzer (aka adversary) selects some linear vector query $f_t: \cX \to \cB_2^{d_t}(1)$. The algorithm has to output an estimate $e_t$ of $f_t(D)$. We say that the algorithm is \emph{$(\alpha, \beta)$-accurate} if, with probability $1 - \beta$, $\|e_t - f_t(D)\|_2 \leq \alpha$ for all $t \in [T]$. Finally, we say that the algorithm satisfies $(\eps, \delta)$-\hybriddp iff the transcript of the interaction satisfies $(\eps, \delta)$-\hybriddp.

\paragraph{Our Algorithm.} The rest of this section is devoted to presenting (and analyzing) our \hybriddp algorithm for OLVQ. The guarantee of the algorithm is stated formally below.

\begin{theorem} \label{thm:vec-sum-apx-dp}
For all $\delta, \beta \in (0, 1/2)$ and $\eps \in (0, \sqrt{\ln(1/\delta)})$,
there is an $(\eps, \delta)$-\hybriddp algorithm for OLVQ 
that is $(\alpha, \beta)$-accurate for $\alpha = O\Bigg(\frac{\sqrt[4]{\ln k \cdot \ln(1/\delta)} \cdot \sqrt{\ln(T n/\beta) + \ln \ln k + \ln \prn{\frac{\sqrt{\ln(1/\delta)}}{\eps}}}}{\sqrt{\eps n}}\Bigg)$.
\end{theorem}

As mentioned earlier, it will be slightly more convenient to work with the zCDP definition instead of DP for composition theorems. In zCDP terms, our algorithm gives the following guarantee: 

\begin{theorem} \label{thm:vec-sum-cdp}
For every $\rho \in (0, 1), \beta \in (0, 1/2)$, there is a $\rho$-\hybridcdp algorithm for OLVQ that is $\prn{\alpha, \beta}$-accurate for $\alpha = O\prn{\frac{\sqrt[4]{\frac{\ln k}{\rho}} \cdot \sqrt{\ln(T n/\beta) + \ln \ln k + \ln(1/\rho)}}{\sqrt{n}}}$.
\end{theorem}

\noindent Note that \Cref{thm:vec-sum-apx-dp} follows from \Cref{thm:vec-sum-cdp} by setting $\rho = \frac{0.1\eps^2}{\log(1/\delta)}$ and applying \Cref{lem:dp-cdp}(ii).
The presentation below follows that of \citet[Section 4.2]{DworkR14} which is based on the original paper of~\cite{HardtR10} and the subsequent work of~\cite{GuptaRU12}. We use the presentation from the Dwork and Roth's book as it uses a more modern privacy analysis through the sparse vector technique, whereas~\cite{HardtR10,GuptaRU12} use a more direct privacy analysis.


\subsection{Linear Vector Query Multiplicative Update}

First, we present the analysis of the multiplicative weight update (MWU) step for linear vector query. This generalizes the standard analysis for scalar-valued query to a vector-valued one. Note that this subsection does not contain any privacy statements, as those will be handled later.

The algorithm takes as input a ``synthetic'' (belief) distribution of the sensitive features for each of the $n$ examples. We write $p^\ell_i$ to denote the distribution for $x^{\priv}_i$. Furthermore, we write $p^\ell_i(y)$ to denote the probability that $x^{\priv}_i = y$ under $p^\ell_i$.
For $\bp^\ell = (p^\ell_i)_{i \in [n]}$ and a linear vector query $f$, we write $f(\bp^\ell; D)$ as a shorthand for $\frac{1}{n} \sum_{i \in [n]} \sum_{y \in \cX^{\priv}} p^\ell_i(y) \cdot f(x^{\pub}_i, y)$. We may drop $D$ for brevity when it is clear from the context. The update is based on the difference between the estimated value  (which will be set as a noised version of the true answer $f(D)$) and $f(\bp^{\ell})$. Since the noise can have unbounded value, we ``truncate'' the dot product when using it to simplify the analysis (recall the notion $\trunc_c$ from \Cref{sec:prelim-concen}). The full update is stated in \Cref{alg:mwu}.

\begin{algorithm}
\caption{$\mwu_{\eta, c}(\bp^{\ell - 1}, f, v, \normbound; D)$ : {\sc Multiplicative Weight Update (MWU) Rule}}
\label{alg:mwu}
\textbf{Input: } Dataset $D = \{x_i\}_{i \in [n]}$, $\bp^{\ell - 1} = (p^{\ell - 1}_i)_{i \in [n]}$, a linear vector query $f$, estimated value $v$,  norm bound $\normbound > 0$ \\
\textbf{Parameters: } Learning rate $\eta > 0$ and truncation bound $c$. \\
$\ophi \gets v - f(\bp^{\ell - 1})$ \hfill $\triangleright$ Difference between evaluated and estimated value \\
$\phi \gets \ophi / \normbound$ \\
\For{$i \in [n]$}{
\For{$y \in \cX^{\priv}$}{
$p_i^\ell(y) \gets \frac{p_i^{\ell - 1}(y) \cdot \exp\prn{\eta \cdot \trunc_c\prn{\ang{\phi, f(x^\pub_i, y)}}}}{\sum_{y' \in \cX^\priv} p_i^{\ell - 1}(y') \cdot \exp\prn{\eta \cdot \trunc_c\prn{\ang{\phi, f(x^\pub_i, y')}}}}$\; \hfill $\triangleright$ Multiplicative Weight Update
}
}
\Return $\bp^\ell = (p^\ell_i)_{i \in [n]}$
\end{algorithm}

We now analyze this update rule. To do so, recall the notion $\clip$ from \Cref{sec:prelim-concen}; it will be convenient to also define the following additional notation:
\begin{align*}
f^{\clip, \phi, c}(x^\pub, y) &:= \clip_{\phi, c}(f(x^{\pub}, y)),
\quad\quad
f^{\clip, \phi, c}(D) := \frac{1}{n} \sum_{i \in [n]} f^{\clip, \phi, c}(x_i), \\
f^{\clip, \phi, c}(\bp^\ell; D) &:= \frac{1}{n} \sum_{i \in [n]} \sum_{y \in \cX^{\priv}} p^\ell_i(y) \cdot f^{\clip, \phi, c}(x^\pub_i, y).
\end{align*}
For readability, we sometimes drop $\phi$ and $c$ from the notations above when it is clear from context.

For convenience, we separate the requirement for the MWU analysis into the following condition. The first item states that the error is sufficiently large, the second that the noise added to $v$ is sufficiently small, the next two assert that clipping does not change the function value too much (for the true answer and that evaluated from the synthetic data $\bp^{\ell - 1}$, respectively), and the remaining two state that $\normbound$ is a good estimate for $\|f(D) - f(\bp^{\ell - 1})\|_2$.

\begin{condition} \label{cond:mwu-analysis}
Suppose that $\eta \leq \frac{1}{c}$ and the following hold:
\begin{enumerate}[(i), nosep]
\item $\|f(D) - f(\bp^{\ell - 1})\|_2 \geq (2c^2 + 7)\eta$, \label{cond:overall-err-large}
\item $\ang{f(D) - v, f(\bp^{\ell - 1}) - f(D)} \leq \eta \cdot \|f(D) - f(\bp^{\ell - 1})\|_2$, \label{cond:noise-dir-small}
\item $\left|\ang{v - f(\bp^{\ell - 1}), f^{\clip}(D) - f(D)}\right| \leq \eta^2$ \label{cond:clip-err-true},
\item $\left|\ang{v - f(\bp^{\ell - 1}), f^{\clip}(\bp^{\ell - 1}) - f(\bp^{\ell - 1})}\right| \leq \eta^2$, \label{cond:clip-err-sync}
\item $\normbound \geq \eta$, \label{cond:norm-bound-lb}
\item $\normbound \leq 2 \cdot \|f(D) - f(\bp^{\ell - 1})\|_2$. \label{cond:norm-bound-accurate}
\end{enumerate}
\end{condition}

Under the above conditions, we show that the update cannot be applied too many times:

\begin{theorem}[MWU Utility Analysis] \label{thm:mwu-util}
Suppose that $\mwu_{\eta, c}(\bp^{\ell - 1}, f, v, \normbound; D)$ is applied for $\ell = 1, \dots, L$ with the initial distribution being the uniform distribution (i.e., $p^0_i(y) = \frac{1}{k}$ for all $i \in [n]$ and $y \in \cX^{\priv}$) such that \Cref{cond:mwu-analysis} holds for all $\ell \in [L]$. Then, it must be that $L < \ln k / \eta^2$.
\end{theorem}

Let the potential be $\Psi^\ell := \frac{1}{n} \sum_{i \in [n]} \ln\prn{\frac{1}{p_i^\ell\prn{x^{\priv}_i}}}$.
The main lemma underlying the proof of \Cref{thm:mwu-util} is that the potential always decreases under \Cref{cond:mwu-analysis}, which immediately implies the proof since the potential satisfies $\Psi^0 = \ln k$ and $\Psi^L > 0$.

\begin{lemma} \label{lem:potential-change}
Assuming that \Cref{cond:mwu-analysis} holds, then $\Psi^{\ell - 1} - \Psi^\ell \geq \eta^2$.
\end{lemma}

To prove \Cref{lem:potential-change}, we use the following two simple facts.

\begin{fact} \label{fact:exp-order}
(i)
For all $x \in \R$, $1 + x \leq \exp(x)$.
(ii) For all $x \in (-\infty, 1]$, $\exp(x) \leq 1 + x + x^2$.
\end{fact}

\begin{proofof}[\Cref{lem:potential-change}]
From the definition of $\clip$ and $\trunc$, we have $\trunc_c\prn{\ang{\phi, f(x^\pub_i, y)}} = \ang{\phi, f^{\clip}(x^\pub_i, y)}$. In other words, the update rule can be rewritten as
\begin{align*}
p_i^\ell(y) \gets \frac{p_i^{\ell - 1}(y) \cdot \exp\prn{\eta \cdot \ang{\phi, f^{\clip}(x^\pub_i, y)}}}{\sum_{y' \in \cX^\priv} p_i^{\ell - 1}(y') \cdot \exp\prn{\eta \cdot \ang{\phi, f^{\clip}(x^\pub_i, y')}}}.
\end{align*}

For brevity, let $\gamma_i^\ell$ be the normalization factor $\sum_{y' \in \cX^\priv} p_i^{\ell - 1}(y') \cdot \exp\prn{\eta \cdot \ang{\phi, f^{\clip}(x^\pub_i, y')}}$ for all $i \in [n]$. We have
\begin{align}
\Psi^{\ell - 1} - \Psi^{\ell} &= \frac{1}{n} \sum_{i \in [n]} \ln\prn{\frac{p_i^{\ell}(x^{\priv}_i)}{p_i^{\ell - 1}(x^{\priv}_i)}}
= \frac{1}{n} \sum_{i \in [n]}  \prn{\eta \cdot \ang{\phi, f^{\clip}(x_i)}  - \ln \gamma_i^\ell} \nonumber \\
&= \eta \cdot \ang{\phi, f^{\clip}(D)} - \frac{1}{n} \sum_{i \in [n]} \ln \gamma_i^\ell. \label{eq:expand-potential-diff}
\end{align}

By definition, $\ang{\phi, f^{\clip}(x_i, y')} \leq c$.
Thus, by our assumption that $\eta \leq 1/c$, we can bound the normalization factor $\gamma_i^\ell$ as follows:
\begin{align*}
\gamma_i^\ell &= \sum_{y' \in \cX^\priv} p_i^{\ell - 1}(y') \cdot \exp\prn{\eta \cdot \ang{\phi, f^{\clip}(x_i, y')}} \\
(\text{\Cref{fact:exp-order}(ii)}) &\leq \sum_{y' \in \cX^\priv} p_i^{\ell - 1}(y') \prn{1 + (\eta \cdot \ang{\phi, f^{\clip}(x_i, y')}) + \prn{\eta \cdot \ang{\phi, f^{\clip}(x_i, y')}}^2} \\
&\leq \sum_{y' \in \cX^\priv} p_i^{\ell - 1}(y') \prn{1 + (\eta \cdot \ang{\phi, f^{\clip}(x_i, y')}) + c^2 \eta^2} \\
&= 1 + \eta \cdot \ang{\phi, \sum_{y' \in \cX^\priv} p_i^{\ell - 1}(y') \cdot f^{\clip}(x_i, y')} + c^2 \eta^2.
\end{align*}

Applying \Cref{fact:exp-order}(i), we can then conclude that
\begin{align*}
\ln \gamma_i^\ell \leq \eta \cdot \ang{\phi, \sum_{y' \in \cX^\priv} p_i^{\ell - 1}(y') \cdot f^{\clip}(x_i, y')} + c^2 \eta^2.
\end{align*}
Taking the average over all $i \in [n]$, we thus have
\begin{align*}
\frac{1}{n} \sum_{i \in [n]} \ln \gamma_i^\ell \leq \eta \cdot \ang{\phi, f^{\clip}(\bp^{\ell - 1})} + c^2 \eta^2.
\end{align*}
Plugging this back into \Cref{eq:expand-potential-diff}, we get
\begin{align*}
&\Psi^{\ell - 1} - \Psi^{\ell} \\
&\geq \eta \cdot \ang{\phi, f^{\clip}(D) - f^{\clip}(\bp^{\ell - 1})}  - c^2 \eta^2 \\
&= \frac{\eta}{\normbound} \cdot \prn{\ang{\ophi, f^{\clip}(D) - f(D)} + \ang{\ophi, f(\bp^{\ell - 1}) - f^{\clip}(\bp^{\ell - 1})} + \ang{\ophi, f(D) - f(\bp^{\ell - 1})} }  - c^2 \eta^2 \\
&\overset{(\spadesuit)}{\geq} \frac{\eta}{\normbound} \cdot \ang{\ophi, f(D) - f(\bp^{\ell - 1})}  - (c^2 + 2) \eta^2 \\
&= \frac{\eta}{\normbound} \cdot \prn{\|f(D) - f(\bp^{\ell - 1})\|_2^2 - \ang{f(D) - v, f(D) - f(\bp^{\ell - 1})}}  - (c^2 + 2) \eta^2 \\
&\overset{(\blacksquare)}{\geq} \frac{\eta}{\normbound} \cdot \prn{\|f(D) - f(\bp^{\ell - 1})\|_2^2 - \eta \cdot \|f(D) - f(\bp^{\ell - 1})\|_2}  - (c^2 + 2) \eta^2 \\
&\overset{(\clubsuit)}{\geq} \frac{\eta}{2 \cdot \|f(D) - f(\bp^{\ell - 1})\|_2} \cdot \prn{(2c^2 + 7)\eta \cdot \|f(D) - f(\bp^{\ell - 1})\|_2 - \eta \cdot \|f(D) - f(\bp^{\ell - 1})\|_2} \\ &\qquad - (c^2 + 2) \eta^2 \\
&= \eta^2,
\end{align*}
where $(\spadesuit)$ follows from \Cref{cond:mwu-analysis}\ref{cond:clip-err-true},\ref{cond:clip-err-sync} and \ref{cond:norm-bound-lb}, $(\blacksquare)$ follows from \Cref{cond:mwu-analysis}\ref{cond:noise-dir-small}, and $(\clubsuit)$ follows from \Cref{cond:mwu-analysis}\ref{cond:overall-err-large} and \ref{cond:norm-bound-accurate}.
\end{proofof}

\subsection{The Algorithm}

We are now ready to describe our algorithm and prove \Cref{thm:vec-sum-cdp}.

\begin{algorithm}[h]
\small
\caption{{\sc Private Vector Multiplicative Weight (PVMW)}}
\label{alg:iterative-const}
\textbf{Input: } Dataset $D$, (online) stream of linear vector queries $f_1, \dots, f_T$ \\
\textbf{Parameters: } Privacy parameter $\rho > 0$, target accuracy $\tau$, maximum number of MWU applications $\itermax$, truncation bound $c$, budget split parameter $\zeta \in (0, 1)$. \\
$\sigma \gets \sqrt{\frac{2\itermax}{(1 - \zeta)\rho}}$ \hfill $\triangleright$ Gaussian Noise Multiplier \\
$\eps' \gets \sqrt{\frac{\zeta\rho}{\itermax}}$ \hfill $\triangleright$ Privacy parameter for AboveThreshold and Norm Esimation \\ 
\For{$i=1,\dots,n$}{
    $p_i^0 \gets$ uniform distribution over $\cX^\priv$\; \hfill $\triangleright$ Initial distribution
}
$\ell \gets 1$\; \hfill $\triangleright$ Counter for \# of updates performed\\
Sample $\chi_\ell \sim \Lap\prn{\frac{4}{\eps'n}}$\; \hfill $\triangleright$ Threshold noise for AboveThreshold \\ 
\For{$t=1,\dots, T$}{
    \While{$\ell < \itermax$}{
        Sample $\nu_{t, \ell} \sim \Lap\prn{\frac{8}{\eps'n}}$\; \hfill $\triangleright$ Query noise for AboveThreshold \\
        \If{$\|f_t(\bp^{\ell - 1}) - f_t(D)\|_2 + \nu_{t, \ell} \geq \tau + \chi_\ell$}{
            Sample $z^{\ell - 1} \sim \cN(0, (2\sigma/n)^2 I_d)$\; \\
            $v^{\ell - 1} \gets f_t(D) + z^{\ell - 1}$\; \\
            Sample $\xi^\ell \sim \Lap\left(\frac{2}{\eps' n}\right)$ \; \\
            $\normbound^{\ell - 1} \gets \|f_t(\bp^{\ell - 1}) - f_t(D)\|_2 + \xi^{\ell - 1}$\; \hfill $\triangleright$ Difference Norm Estimation \\
            $\bp^{\ell} \gets \mwu_{\eta, c}(\bp^{\ell - 1}, f_t, v^{\ell - 1}, \normbound^{\ell - 1}; D)$ \hfill $\triangleright$ Multiplicative Weight Update \\
            $\ell \gets \ell + 1$\; \\
            Sample $\chi_\ell \sim \Lap\prn{\frac{4}{\eps'n}}$\; \hfill $\triangleright$ Resample threshold noise
        }\Else{
            Break $\hfill$ $\triangleright$ Below threshold; estimated value is accurate enough
        }
    }
    \If{$\ell \geq \itermax$}{
        Halt and return ``FAIL''
    }\Else{
        \Return $f_t(\bp^{\ell - 1})$\; \hfill $\triangleright$ Output estimate of $f_t(D)$
    }
}
\end{algorithm}

\begin{proofof}[\Cref{thm:vec-sum-cdp}]
\Cref{alg:iterative-const} contains the description of our algorithm.

\paragraph{Privacy Analysis.}
For each fixed $\ell$, the mechanism is exactly a composition of the AboveThreshold mechanism\footnote{See Appendix~\ref{app:add-prelim} for more explanation on the AboveThreshold, Laplace, and Gaussian mechanisms.} \cite[Algorithm 1]{DworkR14}, the Laplace mechanism with noise multiplier\footnote{The sensitivity of $\|f_t(\bp^{\ell - 1}) - f_t(D)\|_2$ with respect to \hybriddp is $\frac{2}{n}$.} $\frac{1}{\eps'}$ and the Gaussian mechanism with noise multiplier\footnote{The $\ell_2$-sensitivity of $f(D)$ with respect to \hybriddp is $\frac{2}{n}$.} $\sigma$. The first is $\eps'$-\hybriddp \cite[Theorem 3.23]{DworkR14} and, by \Cref{lem:dp-cdp}(i) is thus $(0.5\eps'^2)$-\hybridcdp; similarly, the Laplace mechanism is $(0.5\eps'^2)$-\hybridcdp. Meanwhile, the Gaussian mechanism is $(2/\sigma^2)$-zCDP~\citep{BunS16}. Thus, by the composition theorem (\Cref{lem:composition}) for a fixed $\ell$, the mechanism is $(0.5\eps'^2) + (0.5\eps'^2) + (2/\sigma^2) = (\rho / \itermax)$-\hybridcdp. Thus, applying the composition theorem (\Cref{lem:composition}) across all $\itermax$ iterations, the entire algorithm is $\rho$-\hybridcdp.

\paragraph{Utility Analysis.} 
We set the parameters as follows:
(i) $\zeta = \frac{1}{2}$, (ii) $c = 3$,
(iii) $\eta$ to be the smallest positive real number such that $\eta > \frac{1000 \sqrt[4]{\frac{\ln k}{\rho}} \cdot \sqrt{\ln\prn{\frac{n T \ln k}{\rho \beta \eta}}}}{\sqrt{n}}$,
(iv) $\tau = 16\eta$,
(v) $\itermax = 1 + \left\lfloor \frac{\ln k}{\eta^2}\right\rfloor$.

Note that we may assume w.l.o.g. that $\eta \leq 0.1$ as otherwise the desired guarantee is trivial (i.e., the algorithm can simply outputs zero always).

By the tail bound of Laplace noise (\Cref{lem:tail}(i)), for a fixed $\ell \in [\itermax]$, the following holds with the probability at least $1 - \frac{0.1\beta}{\itermax}$:
\begin{align}
|\chi_\ell| \leq \frac{4}{\eps' n} \cdot \ln\prn{\frac{20 \itermax}{\beta}} \leq \eta,\label{eq:threshold-noise-bound}
\end{align}
where the second inequality follows from our setting of parameters.

Similarly, for fixed $ \ell \in [\itermax], t \in [T]$, the following holds with  probability at least $1 - \frac{0.1\beta}{\itermax T}$:
\begin{align}
|\nu_{t, \ell}| \leq \frac{8}{\eps' n} \cdot \ln\prn{\frac{20 \itermax T}{\beta}} \leq \eta, \label{eq:query-noise-bound}
\end{align}
and, for a fixed $\ell \in [\itermax]$, the following holds with probability at least $1 - \frac{0.1\beta}{\itermax T}$:
\begin{align}
|\xi^{\ell - 1}| \leq \frac{4}{\eps' n} \cdot \ln\prn{\frac{20 \itermax}{\beta}} \leq \eta. \label{eq:normbound-noise-bound}
\end{align}
Observe that, for a given $\bp^{\ell - 1}$, $\ang{z^{\ell - 1}, f_t(\bp^{\ell - 1}) - f_t(D)}$ is distributed as $\cN(0, (\sigma')^2)$ for $\sigma' = \frac{2\sigma}{n} \cdot \|f_t(\bp^{\ell - 1}) - f_t(D)\|_2$. Thus, by the Gaussian tail bound (\Cref{lem:tail}(ii)) and our setting of parameters, the following holds with probability $1 - \frac{0.2\beta}{\itermax T}$ for fixed $\ell \in [\itermax], t \in [T]$:
\begin{align}
\ang{z^{\ell - 1}, f_t(\bp^{\ell - 1}) - f_t(D)}  \leq \sigma' \cdot \sqrt{2 \ln\prn{\frac{10\itermax T}{\beta}}} \leq \eta \cdot \|f_t(\bp^{\ell - 1}) - f_t(D)\|_2. \label{eq:guassian-dir-bound}
\end{align}
Observe also that $v^{\ell - 1} - f_t(\bp^{\ell - 1}) = f_t(D) - f_t(\bp^{\ell - 1}) + z^{\ell - 1}$ is distributed as $\cN(f_t(D) - f_t(\bp^{\ell - 1}), (\sigma'')^2 I_d)$ where $\sigma'' = 2\sigma/n$. By our choice of parameters, we have $\sigma'' \leq 0.1/\log\left(\frac{10 \itermax}{\beta \eta}\right)$ As a result, we can apply \Cref{lem:clip-concen} with $Z = v^{\ell - 1} - f_t(\bp^{\ell - 1})$ and $\cP$ being the uniform distribution over $\{f_t(x_1), \dots, f_t(x_n)\}$. This allows us to conclude that the following holds with probability at least $1 - \frac{0.2 \beta}{\itermax}$ for every $\ell \in [\itermax]$:
\begin{align} \label{eq:clip-err-true-answer}
\left|\ang{v^{\ell - 1} - f_t(\bp^{\ell - 1}), f_t^{\clip}(D) - f_t(D)}\right| \leq 2 \exp\prn{-\frac{0.1}{(\sigma'')^2}} \leq \eta^2.
\end{align}
Similarly, we can apply \Cref{lem:clip-concen} with the same $Z$ but with $\cP$ being the distribution where each $(x^\pub_i, y')$ has probability mass $\frac{p^{\ell - 1}_i(y')}{n}$ to conclude that the following holds with probability at least $1 - \frac{0.2 \beta}{\itermax}$ for every $\ell \in [\itermax]$:
\begin{align} \label{eq:clip-err-syn-data}
\left|\ang{v^{\ell - 1} - f_t(\bp^{\ell - 1}), f_t^{\clip}(\bp^{\ell - 1}) - f_t(\bp^{\ell - 1})}\right| \leq 2 \exp\prn{-\frac{0.1}{(\sigma')^2}} \leq \eta^2.
\end{align}
By a union bound, all of \cref{eq:threshold-noise-bound,eq:query-noise-bound,eq:normbound-noise-bound,eq:guassian-dir-bound,eq:clip-err-true-answer,eq:clip-err-syn-data} hold for all $\ell \in [\itermax], t \in [T]$ with probability at least $1 - \beta$. We assume that these inequalities hold throughout the remainder of the analysis.

From \cref{eq:threshold-noise-bound,eq:query-noise-bound}, if we break the loop, we must have
\begin{align*}
\|f_t(\bp^{\ell - 1}) - f_t(D)\|_2 < \tau + \chi_\ell - \nu_{t, \ell} \leq \tau + 2\eta \leq 18\eta.
\end{align*}
This means that whenever the algorithm outputs an estimate, it has $\ell_2$-error of at most $18\eta \leq O\prn{\frac{\sqrt[4]{\frac{\ln k}{\rho}} \cdot \sqrt{\ln(T n/\beta) + \ln \ln k + \ln(1/\rho)}}{\sqrt{n}}}$ as desired. As a result, it suffices to show that the algorithm never outputs ``FAIL''.

On the other hand, if $\mwu$ is called, we must have
\begin{align} \label{eq:diff-norm-lb}
\|f_t(\bp^{\ell - 1}) - f_t(D)\|_2 \geq \tau + \chi_\ell - \nu_{t, \ell} \geq \tau - 2\eta \geq 14\eta,
\end{align}
implying \Cref{cond:mwu-analysis}\ref{cond:overall-err-large} (for $v = v^{\ell - 1}, f = f_t, c = 3$). Moreover, \cref{eq:guassian-dir-bound,eq:clip-err-true-answer,eq:clip-err-syn-data} are exactly equivalent to \Cref{cond:mwu-analysis}\ref{cond:noise-dir-small}\ref{cond:clip-err-true}\ref{cond:clip-err-sync} respectively. Furthermore, by \cref{eq:diff-norm-lb,eq:normbound-noise-bound}, we also have
\begin{align*}
\normbound^{\ell - 1} = \|f_t(\bp^{\ell - 1}) - f_t(D)\|_2 + \xi^{\ell - 1} \geq 14\eta - \eta = 13\eta,
\end{align*}
and
\begin{align*}
\normbound^{\ell - 1} = \|f_t(\bp^{\ell - 1}) - f_t(D)\|_2 + \xi^{\ell - 1} \leq \|f_t(\bp^{\ell - 1}) - f_t(D)\|_2 + \eta < 2 \cdot \|f_t(\bp^{\ell - 1}) - f_t(D)\|_2,
\end{align*}
which mean that \Cref{cond:mwu-analysis}\ref{cond:norm-bound-lb}\ref{cond:norm-bound-accurate} hold, respectively. In other words, \Cref{cond:mwu-analysis} holds.

From this, we may apply \Cref{thm:mwu-util} to conclude that the number of applications of $\mwu$ is less than $\frac{\ln k}{\eta^2} \leq \itermax$. Thus, the algorithm never outputs ``FAIL''. This completes the proof of the accuracy guarantee.
\end{proofof}

\section{From Online Linear Vector Queries to Convex Optimization}

In this section, we prove our main results for convex optimization with \hybriddp, as formalized below. We note that here we formulate it as the problem of solving $m$ linear queries problems w.r.t. losses $\ell_1, \dots, \ell_m$. The expected excess risk guarantee is for all of these $m$ problems\footnote{We say that the expected excess risk is $e$ if, for all $i \in [m]$, we have $\E[\cL_i(w_i; D) - \min_{w' \in \cW_i} \cL_i(w'; D)] \leq e$ for all $i \in [m]$ where $\cL_i$ denotes the empirical risk and $w_i$ denotes the output of the algorithm for the $i$th instance.}.

\begin{theorem} \label{thm:convex-main}
Suppose that the loss functions $\ell_1, \dots, \ell_m$ are $G$-Lipschitz and $\cW_1, \dots, \cW_m \subseteq \cB_2(R)$. For every $\delta \in (0, 1/2)$ and $\eps \in (0, \ln(1/\delta))$,
there is an $(\eps, \delta)$-\hybriddp algorithm for ERM problems w.r.t. $\ell_1, \dots, \ell_m$ with expected excess risk \\ $O\prn{RG \cdot \frac{\sqrt[4]{\ln k \cdot \ln(1/\delta)} \cdot \sqrt{\ln(mn) + \ln \ln k + \ln\prn{\frac{\sqrt{\ln(1/\delta)}}{\eps}}}}{\sqrt{\eps n}}}$.
\end{theorem}

\begin{theorem} \label{thm:strongly-convex-smooth-main}
Suppose that the loss functions $\ell_1, \dots, \ell_m$ are $G$-Lipschitz, $\mu$-strongly convex, and $\lambda$-smooth. For every $\delta \in (0, 1/2)$ and $\eps \in (0, \ln(1/\delta))$,
there is an $(\eps, \delta)$-\hybriddp algorithm for ERM problems w.r.t. $\ell_1, \dots, \ell_m$ with expected excess risk \\
$O\prn{\frac{G^2}{\mu} \cdot \frac{\sqrt{\ln k \cdot \ln(1/\delta)} \cdot \prn{\ln(m n \lambda / \mu) + \ln \ln (G/\mu) + \ln \ln k + \ln\prn{\frac{\sqrt{\ln(1/\delta)}}{\eps}}}}{\eps n}}$.
\end{theorem}

As stated in the Introduction, these results are shown via simple applications of known optimization algorithms with approximate gradients. The two cases use slightly different notion of approximate gradients, which we will explain below.

\subsection{Convex Case via Approximate Gradient Oracle}

For the convex case, we use the following definition of approximate gradient oracle.

\begin{definition}[Approximate Gradient Oracle, \cite{dAspremont08}]
For any convex function $F: \cW \to \R$, an \emph{$\xi$-approximate gradient oracle} of $F$ provides $\tg(w)$ for any queried $w \in \cW$ such that the following holds: $|\ang{\tg(w) - \nabla F(w), y - u}| \leq \xi$ for all $y, u \in \cW$.
\end{definition}

Under the above condition, standard gradient descent achieves a similar excess risk to the exact gradient case except that there is an extra $\xi$ term:

\begin{theorem}[{\citet[Theorem 4.5]{FeldmanGV17}}] \label{thm:convex-apx-gradient}
For any $G$-Lipschitz convex function $F: \cW \to \R$ with $\cW \subseteq \cB_2(D)$ and any $q \in \N$, there exists an algorithm that makes $q$ queries to an $\xi$-approximate gradient oracle and achieves an excess risk of $O\prn{\frac{DG}{\sqrt{q}} + \xi}$.
\end{theorem}

We are now ready to prove \Cref{thm:convex-main}.

\begin{proofof}[\Cref{thm:convex-main}]
Let $q = n^2$ and $T = m \cdot q$. We simply run the algorithm from \Cref{thm:convex-apx-gradient} for each $i \in [m]$ and, for the $t$th query to approximate gradient oracle, we invoke algorithm from \Cref{thm:vec-sum-apx-dp} with $f_{q(i - 1) + t}(x) = \frac{1}{G} \nabla \ell_i(w; x)$ and scale the answer back by a factor of $G$. From \Cref{thm:vec-sum-apx-dp}, with probability $1 - \beta$, this is an $(2RG \cdot \alpha)$-approximate  gradient oracle for $\alpha = O\Bigg(\frac{\sqrt[4]{\ln k \cdot \ln(1/\delta)} \cdot \sqrt{\ln(T n/\beta) + \ln \ln k + \ln\prn{\frac{\sqrt{\ln(1/\delta)}}{\eps}}}}{\sqrt{\eps n}}\Bigg)$. When this occurs, \Cref{thm:convex-apx-gradient} implies that the excess risk is at most $\frac{RG}{\sqrt{q}} + 2RG \cdot \alpha \leq O(RG \cdot \alpha)$.
With the remaining probability $\beta$, the excess risk is still at most $RG$. Substituting $\beta = 1/n$, we can conclude that the expected excess risk of this algorithm is at most $O\prn{\frac{1}{n} \cdot RG + RG \cdot \alpha} \leq O(RG \cdot \alpha)$. Finally, since the algorithm is simply a post-processing of the result from applying \Cref{thm:vec-sum-apx-dp}, it is $(\eps, \delta)$-\hybriddp. 
\end{proofof}

\subsection{Strongly Convex and Smooth Case via Inexact Oracle}

For the strongly convex and smooth case, we use the following notion called inexact oracle.

\begin{definition}[Inexact Oracle,~\cite{DevolderGN14,devolder2013first}]
For any convex function $F: \cW \to \R$, a \emph{first-order $(\upsilon, \tlambda, \tmu)$-inexact oracle} of $F$ provides $(\tF(w), \tg(w))$ for any queried $w \in \cW$ such that the following holds for all $w, w' \in \cW$:
\[
\frac{\tmu}{2} \cdot \|w'-w\|_2^2 \leq F(w') - (\tF(w) -  \left<\tg(w), w' - w\right>) \leq \frac{\tlambda}{2} \cdot \|w' - w\|_2^2 + \upsilon.
\]
\end{definition}

Note that if the gradient and function values are exact (i.e., $\tF = F, \tg = \nabla$), then the above condition holds for $\upsilon = 0$ when the function $F$ is $\tmu$-strongly convex and $\tlambda$-smooth.

We use the following relation between the $\ell_2$-error of the gradient estimate and inexact oracle.

\begin{lemma}[{\citet[Section 2.3]{devolder2013first}}] \label{lem:apx-to-inexact}
For any $\mu$-strongly convex and $\lambda$-smooth $F: \cW \to \R$, if $\tg: \cW \to \R^d$ is an oracle such that $\|\tg(w) - \nabla g(w)\|_2 \leq \xi$ for all $w \in \cW$, then there exists $\tF : \cW \to \R$ such that $(\tF, \tg)$ is an $\left(\xi^2\left(\frac{1}{\mu} + \frac{1}{2\lambda}\right), 2\lambda, \mu/2\right)$-inexact oracle.
\end{lemma}

It should be noted that we do not specify the exact $\tF$ precisely because the optimization algorithm we use does not need this either:

\begin{theorem}[{\citet[Theorem 4.11]{FeldmanGV17}}]
\label{thm:opt-inexact-oracle}
For any $G$-Lipschitz convex function $F: \cW \to \R$ with $\cW \subseteq \cB_2(D)$ and any $q \in \N, \alpha > 0$, there exists an algorithm that makes $q$ queries to a first-order $(\upsilon, \tlambda, \tmu)$-inexact oracle and achieves an excess risk of $O\left(\frac{\tlambda R^2}{2} \cdot \exp\left(-\frac{\tmu}{\tlambda} \cdot q\right) + \upsilon\right)$ where $R$ denote the distance of the starting point to the optimum. Furthermore, the algorithm only uses the gradient estimate $\tg$ and does not use the function estimate $\tF$.
\end{theorem}

The proof of \Cref{thm:strongly-convex-smooth-main} is almost the same as that of \Cref{thm:convex-main} except that we now use \Cref{thm:opt-inexact-oracle} (and \Cref{lem:apx-to-inexact}) instead of \Cref{thm:convex-apx-gradient}.

\begin{proofof}[\Cref{thm:strongly-convex-smooth-main}]
Note that from Lipschitzness and strong convexity, we can assume that the domain $\cW_i$ is contained in $\cB_2(R)$ for $R = O(G/\mu)$ for all $i \in [m]$.
Let $q = \left\lceil \frac{40\lambda}{\mu} \cdot \ln\left(\frac{\lambda R^2}{n}\right) \right\rceil$ and $T = m \cdot q$. We simply run the algorithm from \Cref{thm:opt-inexact-oracle} for each $i \in [m]$, and for the $t$th query to approximate gradient oracle, we invoke the algorithm from \Cref{thm:vec-sum-apx-dp} with $f_{q(i - 1) + t}(x) = \frac{1}{G} \nabla \ell_i(w; x)$ and scale the answer back by a factor of $G$. From \Cref{thm:vec-sum-apx-dp}, with probability $1 - \beta$, this oracle has $\ell_2$-error at most $G \cdot \alpha$ for $\alpha = O\left(\frac{\sqrt[4]{\ln k \cdot \ln(1/\delta)} \cdot \sqrt{\ln(T n/\beta) + \ln \ln k + \ln\left(\frac{\sqrt{\ln(1/\delta)}}{\eps}\right)}}{\sqrt{\eps n}}\right)$. By \Cref{lem:apx-to-inexact}, this\footnote{Since the algorithm in \Cref{thm:opt-inexact-oracle} does not use the value from the oracle, we do not need to specify it explicitly.} yields an $(\upsilon, 2\lambda, \mu/2)$-inexact oracle for $\upsilon = O\left((G\alpha)^2 \cdot \left(\frac{1}{\lambda} + \frac{1}{\mu}\right)\right)$. 
When this occurs, \Cref{thm:opt-inexact-oracle} implies that the excess risk is at most $O\left(\frac{\tlambda R^2}{2} \cdot \exp\left(-\frac{\tmu}{\tlambda} \cdot T\right) + \upsilon\right) \leq O(\upsilon)$.
With the remaining probability $\beta$, the excess risk is still at most $DG = O(G^2/\mu)$. Substituting $\beta = 1/n$, we can conclude that the expected excess risk of this algorithm is at most $O\left(\frac{1}{n} \cdot \frac{G^2}{\mu} + \upsilon\right) \leq O(\upsilon)$. Finally, since the algorithm is simply a post-processing of the result from applying \Cref{thm:vec-sum-apx-dp}, it is $(\eps, \delta)$-\hybriddp. 
\end{proofof}

%% file: s4-conclusion.tex
\section{Conclusion and Open Questions}

We gave improved bounds for convex ERM with \hybriddp; crucially they show that the dependency on $k$ is only polylogarithmic instead of polynomial as in previous works. As an intermediate result, we give an algorithm for answering (online) linear vector queries. Given that linear queries are used well beyond convex optimization, we hope that this will find more applications.

An obvious open question is to close the gap between the upper and lower bounds. Another interesting question is to come up with a \emph{pure}-DP algorithm with a similar bound as in \Cref{thm:convex-main}. In particular, it is open if there is any pure-DP algorithm where the error depends only polylogarithmically on $k$.

%% file: a1-app.tex
\section{Additional Preliminaries}
\label{app:add-prelim}

In this section, we give some more background on the DP mechanisms from literature that we use as subroutines. We start with sensitivity, the Gaussian mechanism, and the Laplace mechanism.

\begin{definition}[Sensitivity]
For any query $g: \cX^n \to \R^d$ and $p \geq 1$, its \emph{$\ell_p$-sensitivity} is defined as $\Delta_p(g) := \max_{D, D'} \|g(D) - g(D')\|_p$ where the maximum is over all neighboring datasets $D, D'$.
\end{definition}

\begin{definition}[Gaussian Mechanism]
The \emph{Gaussian mechanism} for a function $g: \cX^n \to \R^d$ simply outputs $g(D) + Z$ on input $D$ where $Z \sim \cN(0, \sigma^2 I_d)$. 
\end{definition}

The zCDP property of Gaussian mechanism is well known:
\begin{theorem}[\cite{BunS16}]
The Gaussian mechanism is $\rho$-zCDP for $\rho = 0.5\Delta_2(g)^2 / \sigma^2$.
\end{theorem}

\begin{definition}[Laplace Mechanism]
The \emph{Laplace mechanism} for a function $g: \cX^n \to \R^d$ simply outputs $g(D) + Z$ on input $D$ where $Z \sim \Lap(a)^{\otimes d}$. 
\end{definition}

The Laplace mechanism has been shown to be DP in the original work of~\cite{dwork2006calibrating}. 
\begin{theorem}[\cite{dwork2006calibrating}]
The Laplace mechanism is $\eps$-DP for $\eps = \Delta_1(g) / a$.
\end{theorem}

Another tool we use is the so-called AboveThreshold mechanism, from the Sparse Vector Technique~\cite{DworkNRRV09}. This mechanism is shown in \Cref{alg:abovethresh}, following the presentation in \cite[Algorithm 1]{DworkR14}.

\begin{algorithm}[H]
\small
\caption{{\sc AboveThreshold}}
\label{alg:abovethresh}
\textbf{Input: } Dataset $D$, threshold $\tau$, (online) stream of queries $g_1, \dots, g_T: \cX^n \to \R$ \\
\textbf{Parameters: } Privacy parameter $\eps$, sensitivity bound $\Delta$ \\
Sample $\chi \sim \Lap\prn{\frac{2\Delta}{\eps}}$\; \hfill $\triangleright$ Threshold noise \\ 
\For{$t=1,\dots, T$}{
Sample $\nu_{t} \sim \Lap\prn{\frac{4 \Delta}{\eps}}$\; \hfill $\triangleright$ Query noise \\
\If{$g_t(D) + \nu_t \geq \tau + \chi$}{
Output $\top$ and halt
}\Else{
Output $\perp$ and continue
}
}
\end{algorithm}

Despite the fact that we handle multiple queries, AboveThreshold only requires a constant amount of noise and satisfies pure-DP:
\begin{theorem}[{\cite[Theorem 3.23]{DworkR14}}]
Suppose that each of $g_1, \dots, g_T$ has sensitivity at most $\Delta$. Then, \textsc{AboveThreshold} (\Cref{alg:abovethresh}) is $\eps$-DP.
\end{theorem}

\section{Proof of \Cref{lem:clip-concen}}
\label{app:concen-proof}

We will use the following tail bounds for Laplace and Gaussian distributions.

\begin{lemma}[Tail Bounds] 
\label{lem:tail}
(i) $\Pr_{X \sim \Lap(a)}[|X| \geq t] \leq 2 \exp(-t/a)$.
(ii) $\Pr_{X \sim \cN(0, \sigma^2)}[|X| \geq t] \leq 2\exp(-0.5(t/\sigma)^2)$.
\end{lemma}

Using the above tail bound, it is relatively simple to show \Cref{lem:clip-concen}.

\begin{proofof}[\Cref{lem:clip-concen}]
By Markov's inequality, it suffices to show that
\begin{align} \label{eq:expectation-clip-concen}
\E_Z\left[\left|\left<Z, \E_{U \sim \cP}[\clip_{Z, 3}(U)] - \mu_{\cP}\right>\right|\right] \leq 4 \exp(-0.2/\sigma_Z^2).
\end{align}
To show this, first observe that
\begin{align}
\E_Z\left[\left|\left<Z, \E_{U \sim \cP}[\clip_{Z, 3}(U)] - \mu_{\cP}\right>\right|\right] &= \E_Z\left[\left|\E_{U \sim \cP}[\left<Z, \clip_{Z, 3}(U)\right> - \left<Z, U\right>]\right|\right] \nonumber \\
&= \E_Z\left[\left|\E_{U \sim \cP}\left[\trunc_{3}(\left<Z, U\right>) - \left<Z, U\right>\right]\right|\right]\nonumber \\ 
&\leq \E_Z\left[\E_{U \sim \cP}\left[\left|\trunc_{3}(\left<Z, U\right>) - \left<Z, U\right>\right|\right]\right]\nonumber \\
&= \E_{U \sim \cP}\left[\E_Z\left[\left|\trunc_{3}(\left<Z, U\right>) - \left<Z, U\right>\right|\right]\right]\nonumber \\
&\leq \sup_{u \in \cB_2^d(1)} \E_Z\left[\left|\trunc_{3}(\left<Z, u\right>) - \left<Z, u\right>\right|\right], \label{eq:guass-func-exp-bound}
\end{align}
where the first inequality follows from Jensen.

Let us now fixed $u \in \cB_2^d(1)$. Observe that
\begin{align}
& \E_Z\left[\left|\trunc_{3}(\left<Z, u\right>) - \left<Z, u\right>\right|\right] 
\leq \E_Z\left[\left|\left<Z, u\right>\right| \cdot \ind[|\left<Z, u\right>| > 3]\right] \nonumber \\
& \leq \sqrt{\E_Z\left[\left<Z, u\right>^2\right] \cdot \E_Z\left[\ind[|\left<Z, u\right>| > 3]\right]} \leq \sqrt{\E_Z\left[\left<Z, u\right>^2\right] \cdot \Pr_Z\left[\ind[|\left<Z, u\right>| > 3]\right]}, \label{eq:trunc-err-lem}
\end{align}
where the second inequality follows from Cauchy--Schwarz.

Notice further that $\left<Z, u\right>$ is distributed as $\cN(\mu', \sigma')$ for $\mu' = \left<\mu_Z, u\right>$ and $\sigma' = \sigma_Z \cdot \|u\|_2 \leq \sigma_Z$. Moreover, since $\mu_Z \in \cB_2^d(2)$, we have $|\mu'| \leq 2$. As a result, its second moment satisfies
\begin{align*}
\E_Z\left[\left<Z, u\right>^2\right] = (\mu')^2 + (\sigma')^2 \leq 2 + \sigma_Z^2 \leq 3.
\end{align*}
Moreover, applying \Cref{lem:tail}(ii), we can conclude that
\begin{align*}
\Pr_Z[|\left<Z, u\right>| > 3] \leq 2\exp\left(-0.5 / \sigma_Z^2\right).
\end{align*}
Plugging these back into \eqref{eq:trunc-err-lem}, we get
\begin{align*}
\E_Z\left[\left|\trunc_{3}(\left<Z, u\right>) - \left<Z, u\right>\right|\right] \leq \sqrt{6\exp\left(-0.5 / \sigma_Z^2\right)} \leq 4 \exp(-0.2/\sigma_Z^2).
\end{align*}
From this and \eqref{eq:guass-func-exp-bound}, we can conclude that \eqref{eq:expectation-clip-concen} holds as desired.
\end{proofof}

%% file: a2-lb.tex
\section{Excess Risk Lower Bound}

In this section, we prove a nearly-matching lower bound on the excess risk.  We will use ``group privacy'' bound in our lower bound proof. For any neighboring relationship $\sim$, we use $\sim_r$ to denote the relationship where two datasets $D, D'$ are considered neighbors if there exists a sequence $D = D_0, D_1, \dots, D_r = D'$ such that $D_{i-1} \sim D_i$ for all $i \in [r]$.

\begin{lemma}[Group Privacy, e.g., {\citet[Lemma 2.2]{Vadhan17}}] \label{thm:group}
Suppose that $\A$ is $(\eps, \delta)$-DP w.r.t. $\sim$, then it is $(\eps',\delta')$-DP w.r.t. $\sim_r$ for $\eps' = r \eps$ and $\delta' = \frac{e^{r\eps}-1}{e^\eps - 1} \cdot \delta$.
\end{lemma}

Our lower bounds are stated formally below.

\begin{theorem} \label{thm:lb-convex}
For any $\eps, \delta, R, G > 0$ and $n, k \in \N$ such that $\eps \leq \ln k$ and $\delta \leq \frac{0.4 \eps}{k}$, there exists a $G$-Lipschitz convex loss function $\ell: \cW \times \cX \to \R$, where $\cW \subseteq \cB_2(R)$, such that any $(\eps, \delta)$-\hybriddp algorithm for ERM w.r.t on $\ell$ has expected excess empirical risk at least $\Omega\left(DG \cdot \min\left\{1, \frac{\sqrt{\log k}}{\sqrt{\eps n}}\right\}\right)$.
\end{theorem}

\begin{theorem} \label{thm:lb-strongly-convex}
For any $\eps, \delta, G, \mu > 0$ and $n, k \in \N$ such that  $\eps \leq \ln k$, $\delta \leq \frac{0.4 \eps}{k}$,
there exists a $G$-Lipschitz $\mu$-strongly convex $\mu$-smooth loss function $\ell: \cW \times \cX \to \R$ such that any $(\eps, \delta)$-\hybriddp algorithm for ERM w.r.t $\ell$ has expected excess empirical risk at least $\Omega\left(\frac{G^2}{\mu} \cdot \min\left\{1, \frac{\log k}{\eps n}\right\}\right)$.
\end{theorem}


\subsection{Convex Case}

To show the convex case (\Cref{thm:lb-convex}), 
it will in fact be convenient to first prove a (smaller) lower bound that holds even against very large $\eps$ (up to $O(\ln k)$), as stated more formally below.

\begin{theorem} \label{thm:convex-lb-large-eps}
For any $\eps, \delta, R, G > 0$ and $n, k \in \N$ such that $\frac{e^{\eps}}{e^\eps + k - 1} + \delta < 0.99$, there exists a $G$-Lipschitz convex loss function $\ell: \cW \times \cX \to \R$, where $\cW \subseteq \cB_2(R)$, such that any $(\eps, \delta)$-\hybriddp algorithm for ERM w.r.t $\ell$ has expected excess empirical risk at least $\Omega\left(\frac{RG}{\sqrt{n}}\right)$.
\end{theorem}

\begin{proof}
Let $\cX^\pub = [n]$, $\cX^\priv = [k]$, $d = nk$, and $\cW = \cB_2^d(R)$. For $x^\pub \in \cX^{\pub}, y \in \cX^\priv$, we write $j(x^\pub, y)$ as a  shorthand for $k(x^\pub - 1) + y$. Finally, let $\ell: \cW \times (\cX^\pub \times \cX^\priv) \to \R$ be
\begin{align*}
\ell(w, (x^\pub, y)) &= - G \cdot \left<w, e_{j(x^\pub, y)}\right>,
\end{align*}
where $e_j \in \R^d$ denotes the $j$th vector in the standard basis for all $j \in [d]$.

Let $D = \{x_i\}_{i \in [n]}$ be the input dataset generated as follows:
\begin{itemize}[nosep]
\item Sample $y_1, \dots, y_n \sim [k]$ independently and uniformly at random and
\item Let $x_i = (i, y_i)$ for all $i \in [n]$.
\end{itemize}
Let $\A: (\cX \times \cY)^n \to \cW$ denote any $(\eps,\delta)$-\hybriddp algorithm.

For $i \in [n]$, we write $w(i)$ as a shorthand for $(w_{(i - 1)k + 1}, \dots, w_{ik})$. We have\footnote{Here ties can be broken arbitrarily for $\argmax$.}
\begin{align*}
&\Pr_{D, \hw \sim \A(D)}\left[\left<\hw(i), e_{y_i}\right> > \frac{1}{\sqrt{2}} \|\hw(i)\|_2\right] \\
&\leq \Pr_{D, \hw \sim \A(D)}\left[y_i = \argmax_{y \in [k]} \left<\hw(i), e_y\right> \right] \\
&\overset{(\blacksquare)}{=} \frac{1}{k} \sum_{y' \in [k]} \Pr_{D, \hw \sim \A(D)}\left[y' = \argmax_{y \in [k]} \left<\hw(i), e_y\right> \mid y_i = y' \right] \\
&= \frac{1}{k} \sum_{y' \in [k]} \Bigg(\frac{e^{\eps}}{e^\eps + k - 1} \cdot \Pr_{D, \hw \sim \A(D)}\left[y' = \argmax_{y \in [k]} \left<\hw(i), e_y\right> \mid y_i = y' \right] \\
&\qquad \qquad \qquad + \frac{k - 1}{e^\eps + k - 1} \cdot \Pr_{D, \hw \sim \A(D)}\left[y' = \argmax_{y \in [k]} \left<\hw(i), e_y\right> \mid y_i = y' \right] \Bigg) \\
&\overset{(\clubsuit)}{\leq} \frac{1}{k} \sum_{y' \in [k]} \Bigg(\frac{e^{\eps}}{e^\eps + k - 1} \cdot \Pr_{D, \hw \sim \A(D)}\left[y' = \argmax_{y \in [k]} \left<\hw(i), e_y\right> \mid y_i = y' \right] \\
&\qquad \qquad \qquad + \frac{k - 1}{e^\eps + k - 1} \left(e^{\eps} \cdot \Pr_{D, \hw \sim \A(D)}\left[y' = \argmax_{y \in [k]} \left<\hw(i), e_y\right> \mid y_i \ne y' \right] + \delta\right) \Bigg) \\
&\leq \frac{1}{k} \sum_{y' \in [k]} \left(\frac{e^{\eps}}{e^\eps + k - 1} \cdot \Pr_{D, \hw \sim \A(D)}\left[y' = \argmax_{y \in [k]} \left<\hw(i), e_y\right>\right] + \delta \right) \\
&= \frac{e^{\eps}}{e^\eps + k - 1} + \delta,
\end{align*}
where $(\blacksquare)$ follows from $y_i \sim y'$ and $(\clubsuit)$ follows from the $(\eps, \delta)$-\hybriddp guarantee of $\A$.

Now, let $I_{\hw, D}$ denote the set $\{i \in [n] \mid \left<\hw(i), e_{y_i}\right> > \frac{1}{\sqrt{2}} \|\hw(i)\|_2\}$. The above inequality implies that
\begin{align} \label{eq:good-coor-bound}
\E_{D, \hw \sim \A(D)}[|I_{\hw, D}|] \leq \left(\frac{e^{\eps}}{e^\eps + k - 1} + \delta\right) n \leq 0.99n,
\end{align}
where the inequality is from our assumption on $\eps, \delta, k$.

Meanwhile, $|I_{\hw, D}|$ can be used to bound the loss function as follows.
\begin{align*}
\cL(w; D) &= \frac{1}{n} \sum_{i \in [n]} \ell(w, (i, y_i)) \\
&= \frac{-G}{n} \sum_{i \in [n]} \left<w, e_{j(i,y_i)}\right> \\
&= \frac{-G}{n} \sum_{i \in [n]} \left<w(i), e_{y_i}\right> \\
&= \frac{-G}{n} \left[\left(\sum_{i \in I_{w, D}} \left<w(i), e_{y_i}\right>\right) + \left(\sum_{i \notin I_{w, D}} \left<w(i), e_{y_i}\right>\right)\right] \\
&\geq \frac{-G}{n} \left[\left(\sum_{i \in I_{w, D}} \|w(i)\|_2\right) + \left(\sum_{i \notin I_{w, D}} \frac{1}{\sqrt{2}} \|w(i)\|_2\right)\right] \\
&\overset{(\blacktriangle)}{\geq} \frac{-L}{n} \left[\sqrt{\left(\sum_{i \in I_{w, D}} 1\right) + \left(\sum_{i \notin I_{w, D}} \frac{1}{2}\right)} \cdot \sqrt{\left(\sum_{i \in I_{w, D}} \|w(i)\|^2_2\right) + \left(\sum_{i \notin I_{w, D}} \|w(i)\|^2_2\right)}\right] \\
&= \frac{-G}{n} \cdot \sqrt{n/2 + |I_{w, D}|/2} \cdot \|w\|_2 \\
&\geq \frac{-RG}{n} \cdot \sqrt{n/2 + |I_{w, D}|/2},
\end{align*}
where $(\blacktriangle)$ follows from Cauchy--Schwarz.

Note that, by picking $w^* = \frac{R}{\sqrt{n}} \sum_{i \in [n]} e_{j(i,y_i)}$, we have
\begin{align*}
\cL(w^*; D) = -\frac{RG}{\sqrt{n}}.
\end{align*}
Thus, the excess risk is
\begin{align*}
\cL(w; D) - \cL(w^*; D) \geq \frac{RG}{n} \left(\sqrt{n} - \sqrt{n/2 + |I_{w, D}|/2}\right).
\end{align*}
As a result, the expected excess risk of $\A$ is
\begin{align*}
\E_{D, \hw \sim \A(D)}[\cL(\hw, D) - \cL(w^*, D)] &\geq \E_{D, \hw \sim \A(D)}\left[\frac{RG}{n} \left(\sqrt{n} - \sqrt{n/2 + |I_{\hw, D}|/2}\right)\right] \\
&\geq \frac{RG}{n} \left(\sqrt{n} - \sqrt{n/2 + \E_{D, \hw \sim \A(D)}[|I_{\hw, D}|]/2}\right) \\
&\overset{\eqref{eq:good-coor-bound}}{\geq} \frac{RG}{n} \left(\sqrt{n} - \sqrt{0.995 n}\right) \\
&\geq \Omega\left(\frac{RG}{\sqrt{n}}\right). 
\end{align*}
\end{proof}

\Cref{thm:lb-convex} now easily follows from applying the group privacy bound (\Cref{thm:group}).

\begin{proofof}[\Cref{thm:lb-convex}]
Let $r = \left\lfloor \frac{\ln k}{\eps} \right\rfloor$.
We will henceforth assume that $n \geq r$; otherwise, we can instead apply a lower bound for largest $k'$ such that $\frac{\log k'}{\eps} \leq n$ instead (which would give a lower bound of $\Omega(RG)$ already).

Suppose for the sake of contradiction that there is an $(\eps, \delta)$-\hybriddp algorithm $\A$ that yields $o\left(\frac{RG \sqrt{\log k}}{\sqrt{\eps n}}\right)$ excess risk in the aforementioned setting in the theorem statement.
We assume w.l.o.g.\footnote{Otherwise, we may simply add dummy input points with constant loss functions.} that $n$ is divisible by $r$; let $n' = n/r$.

Let $\A'$ be an algorithm that takes in $n'$ points, replicates each input datapoint $r$ times and then runs $\A$. From \Cref{thm:group}, $\A'$ is $(\eps',\delta')$-\hybriddp for $\eps' = r\eps$ and $\delta' = \frac{e^{r\eps}-1}{e^\eps - 1} \cdot \delta$. The expected excess risk of $\A'$ is  
\begin{align*}
o\left(\frac{RG\sqrt{\log k}}{\sqrt{\eps n}}\right) = o\left(\frac{RG}{\sqrt{n'}}\right).
\end{align*}
Furthermore, we have
\begin{align*}
\frac{e^{\eps'}}{e^{\eps'} + k - 1} + \delta' = \frac{e^{r\eps}}{e^{r\eps} + k - 1} + \frac{e^{r\eps} - 1}{e^{\eps} - 1} \cdot \delta \leq \frac{k}{2k - 1} + \frac{k}{\eps} \cdot \delta \leq 0.9.
\end{align*}
This contradicts \Cref{thm:convex-lb-large-eps}.
\end{proofof}

\subsection{Strongly Convex (and Smooth) Case}
\label{app:strongly-convex}

The strongly convex and smooth case proceeds in very much the same way except we use the squared loss instead. 

\begin{theorem} \label{thm:strongly-convex-lb-large-eps}
For any $\eps, \delta, G, \mu > 0$ and $n, k \in \N$ such that $\frac{e^{\eps}}{e^\eps + k - 1} + \delta < 0.99$, there exists an $G$-Lipschitz $\mu$-strongly convex $\mu$-smooth loss function $\ell: \cW \times \cX \to \R$ such that any $(\eps, \delta)$-\hybriddp algorithm for ERM w.r.t $\ell$ has expected excess empirical risk at least $\Omega\left(\frac{G^2}{\mu n}\right)$.
\end{theorem}

\begin{proof}
We use the same notation as in the proof of \Cref{thm:convex-lb-large-eps} with $R = 0.5 G/\mu$, except that we let the loss function be
 \begin{align*}
\ell(w, (x, y)) &= \frac{\mu}{2} \left\| w - R \cdot e_{j(x,y)} \right\|_2^2.
\end{align*}
It is simple to see that $\ell$ is $\mu$-strongly convex, $\mu$-smooth, and $G$-Lipschitz.

Similar to the proof of \Cref{thm:convex-lb-large-eps}, we can prove \eqref{eq:good-coor-bound}. From this, we rearrange the excess risk (where $w^* := \frac{R}{n} \sum_{i \in [n]} e_{j(i, y_i)}$) as follows:
\begin{align*}
\cL(w; D) - \cL(w^*; D) 
&= \frac{\mu}{2} \|w - w^*\|_2^2 \\
&= \frac{\mu}{2} \sum_{i \in [n]} \left\|w(i) - \frac{R}{n} \cdot e_{y_i}\right\|_2^2 \\
&\geq \frac{\mu}{2} \sum_{i \notin I_{w, D}} \left\|w(i) - \frac{R}{n} \cdot e_{y_i}\right\|_2^2 \\
&= \frac{\mu}{2} \sum_{i \notin I_{w, D}} \left(\left\|w(i)\right\|^2 - \frac{2R}{n} \left<w(i), e_{y_i}\right> + \frac{R^2}{n^2}\right) \\
&\overset{(\spadesuit)}{\geq}  \frac{\mu}{2} \sum_{i \notin I_{w, D}} \left(\left\|w(i)\right\|_2^2 - \frac{R\sqrt{2}}{n} \cdot \left\|w(i)\right\|_2 + \frac{R^2}{n^2}\right) \\
&=  \frac{\mu}{2} \sum_{i \notin I_{w, D}} \left(\left(\left\|w(i)\right\|_2 - \frac{R}{n\sqrt{2}}\right)^2 + \frac{R^2}{2n^2}\right) \\
&\geq \frac{\mu R^2}{4n^2} \cdot (n - |I_{w, D}|),
\end{align*}
where $(\spadesuit)$ follows from the definition of $I_{w, D}$.

Thus, we have
\begin{align*}
\E_{D, \hw \sim \A(D)}[\cL(\hw, D) - \cL(w^*, D)] \geq \frac{\mu R^2}{4n^2} \left(n - \E_{D, \hw \sim \A(D)}[|I_{\hw, D}|]\right) \overset{\eqref{eq:good-coor-bound}}{\geq} \Omega\left(\frac{\mu R^2}{n}\right) \geq \Omega\left(\frac{G^2}{\mu n}\right).
\end{align*}
\end{proof}

Again, \Cref{thm:lb-strongly-convex} easily follows via group privacy.

\begin{proofof}[\Cref{thm:lb-strongly-convex}]
Let $r = \left\lfloor \frac{\ln k}{\eps} \right\rfloor$.
We will henceforth assume that $n \geq r$; otherwise, we can instead apply a lower bound for smallest $k'$ such that $\frac{\log k'}{\eps} \leq n$ instead (which gives a lower bound of $\Omega(G^2/\mu)$ already).

Suppose for the sake of contradiction that there is an algorithm $\A$ that yields $o\left(\frac{G^2}{\mu} \cdot \frac{\log k}{\eps n}\right)$ excess risk in the aforementioned setting in the theorem statement.
We assume w.l.o.g. that $n$ is divisible by $r$; let $n' = n/r$.

Let $\A'$ be an algorithm that takes in $n'$ points, replicates each input data point $r$ times and then runs $\A$. From \Cref{thm:group}, $\A'$ is $(\eps',\delta')$-\hybriddp for $\eps' = r\eps$ and $\delta' = \frac{e^{r\eps}-1}{e^\eps - 1} \cdot \delta$. The expected excess risk of $\A'$ is  
\begin{align*}
o\left(\frac{G^2}{\mu} \cdot \frac{\log k}{\eps n}\right) = o\left(\frac{G^2}{\mu} \cdot \frac{1}{n'}\right).
\end{align*}
Similar to the calculation in the proof of \Cref{thm:lb-convex}, we have $\frac{e^{\eps'}}{e^{\eps'} + k - 1} + \delta' \leq 0.9$. 
Thus, this contradicts \Cref{thm:convex-lb-large-eps}.
\end{proofof}